\documentclass[letterpaper, 10 pt, conference]{ieeeconf}

\IEEEoverridecommandlockouts                              
\newif\ifcomment
\commenttrue   

\usepackage{amsmath, amsfonts} 
\usepackage{algorithm}
\usepackage{algpseudocode}
\usepackage{mathtools}

\usepackage{graphicx}
\usepackage{siunitx}
\usepackage{subcaption}
\usepackage{booktabs}
\usepackage{xcolor}
\usepackage{textcomp}
\usepackage{gensymb} 
\usepackage{multirow} 
\usepackage{comment}
\usepackage{verbatim}
\usepackage[skip=0pt, font=footnotesize]{caption}
\usepackage{bm}
\usepackage[noadjust]{cite}
\usepackage{makecell}
\usepackage{tabularx}
\usepackage{pifont}
\usepackage{layouts}

\usepackage{physics}

\usepackage{amsthm} 
\usepackage{lipsum}

\usepackage{amssymb}
\usepackage[english]{babel}
\usepackage[hidelinks]{hyperref}

\newtheoremstyle{colon}%
{}
{}
{\itshape}
{}
{\bfseries}
{:}
{ }
{}
\theoremstyle{colon}
\newtheorem{theorem}{Theorem}

\newtheorem{lemma}{Lemma}
\newtheorem{problem}{Problem}
\newtheorem{model}{Model}

\theoremstyle{remark}

\IEEEaftertitletext{\vspace{-0.5\baselineskip}}

\makeatletter
\renewcommand{\@IEEEsectpunct}{ \,}
\makeatother

\makeatletter 
\newcommand\Label[1]{&\refstepcounter{equation}(\theequation)\ltx@label{#1}&}
\makeatother

\definecolor{peru}{rgb}{0.803921568627451, 0.5215686274509804, 0.24705882352941178}
\definecolor{violet}{rgb}{0.9333333333333333, 0.5098039215686274, 0.9333333333333333}
\definecolor{greeN}{rgb}{0.17254901960784313, 0.6274509803921569, 0.17254901960784313}
\definecolor{stage0}{RGB}{187,248,255}
\definecolor{stage1}{RGB}{250,255,187}
\definecolor{stage2}{RGB}{187,255,196}
\definecolor{stage0_dark}{RGB}{0,180,200}
\definecolor{stage1_dark}{RGB}{200,180,0}
\definecolor{stage2_dark}{RGB}{0,200,0}

\definecolor{centerline}{RGB}{51,51,255}
\definecolor{exterior}{RGB}{255,153,51}
\definecolor{interior}{RGB}{0,153,0}
\definecolor{slalom}{RGB}{255,51,51}

\def\anonymous{1} 

\newcommand\ringring[1]{%
  {
   \mathop{\kern0pt #1}\limits^{
     \vbox to-1.85ex{
       \kern-2ex 
       \hbox to 0pt{\hss\normalfont\kern.1em \r{}\kern-.45em \r{}\hss}%
       \vss 
     }
   }
  }
}

\renewcommand{\baselinestretch}{0.978}

\newcolumntype{M}[1]{>{\centering\arraybackslash}m{#1}}

\makeatletter
\def\endthebibliography{%
	\def\@noitemerr{\@latex@warning{Empty `thebibliography' environment}}%
	\endlist
}
\makeatother

\IEEEoverridecommandlockouts
\usepackage{tikz}
\usepackage{textcomp}
\usepackage{hyperref}
\usepackage{lipsum}

\newcommand\copyrighttext{%
	\footnotesize Published in IEEE Conference on Decision and Control (CDC), Singapore, 2023.\newline
	 \textcopyright 2023 IEEE. Personal use of this material is permitted.
	Permission from IEEE must be obtained for all other uses, in any current or future media, including reprinting/republishing this material for advertising or promotional purposes, creating new collective works, for resale or redistribution to servers or lists, or reuse of any copyrighted component of this work in other works.}
\newcommand\copyrightnotice{%
	\begin{tikzpicture}[remember picture,overlay]
		\node[anchor=south,yshift=10pt] at (current page.south) {\fbox{\parbox{\dimexpr\textwidth-\fboxsep-\fboxrule\relax}{\copyrighttext}}};
	\end{tikzpicture}%
}

\title{\LARGE \bf
Pose-Following with Dual Quaternions}
\if\anonymous1
\author{Jon Arrizabalaga$^{1}$ and Markus Ryll$^{1,2}$
	\thanks{$^{1}$Autonomous Aerial Systems, School of Engineering and Design,  Technical University of Munich, Germany. E-mail: {\tt\small jon.arrizabalaga@tum.de} and {\tt\small markus.ryll@tum.de}}%
	\thanks{$^{2}$Munich Institute of Robotics and Machine Intelligence (MIRMI), Technical University of Munich}
}
\fi
\begin{document}

\maketitle
\ifcomment
    \copyrightnotice
\fi
\begin{abstract}
This work focuses on pose-following, a variant of path-following in which the goal is to steer the system’s position and attitude along a path with a moving frame attached to it. Full body motion control, while accounting for the additional freedom to self-regulate the progress along the path is an appealing trade-off. Towards this end, we extend the well-established dual quaternion based pose-tracking method into a pose-following control law. Specifically, we derive the equations of motion for the full pose error between the geometric reference and the rigid body in the form of a dual quaternion and dual twist, and subsequently, formulate an almost globally asymptotically stable control law. The global attractivity of the presented approach is validated in a spatial example, while its benefits over pose-tracking are showcased through a planar case-study.
\end{abstract}
\begin{flushleft}
\textbf{Code}: \url{https://github.com/jonarriza96/pfdq}\\
\textbf{Video}: \url{https://youtu.be/TQig2j90Ijc}
\end{flushleft}
\vspace{-1mm}
\section{INTRODUCTION}
\noindent Simultaneous attitude and position -- \emph{pose} -- control of a rigid body in three-dimensional space is fundamental to many applications, such as for autonomous vehicles, spacecrafts or robotic manipulators. 

The simplest approach to address the pose control problem is decoupling it into two separate subproblems~\cite{fjellstad1994position,stansbery2000position}. On the one hand, a \emph{position controller} drives the translational motions, and on the other hand, an \emph{attitude controller} regulates the rotational behavior. This separation relates to the de facto representation of the rigid body dynamics, in which the translational and angular motions are expressed separately (as planar and spatial examples, see eq. 7 in \cite{arrizabalaga2021caster} and eq. 1 in \cite{arrizabalaga2022towards}). However, such partitioning poses a challenge to effectively control the interdependence between the rotational and translational dynamics.

An alternative to this decoupling is representing the system dynamics globally on the configuration manifold of the special Euclidean group $\mathrm{SE}(3)$. Doing so allows for leveraging the group structure to first avoid singularities and second extend proportional derivative (PD) feedback controllers, for the pose-tracking problem~\cite{bullo1995proportional}. Control methods inspired by these findings have shown very promising results within a plethora of robotic platforms, such as quadrotors~\cite{lee2010geometric}, robotic manipulators~\cite{figueredo2013robust}, walking robots~\cite{bledt2018cheetah} and spacecrafts~\cite{filipe2015adaptive}. 

To represent a rigid body in $\mathrm{SE}(3)$, it is customary to combine a three-dimensional vector of the Cartesian coordinates with either a rotation matrix -- resulting in a homogeneous transformation matrix (HTM) -- or a unit quaternion. A less common choice are \emph{unit dual quaternions}. 

Unit quaternion parameterization can be summarized into five advantages: First, when compared to HTMs, unit dual quaternions offer a more compact representation of $\mathrm{SE}(3)$, as it requires only eight parameters (against twelve) to describe the motions of a rigid body -- rotations and translations -- in a singularity-free and global manner~\cite{yang1963application}. Second, dual quaternion multiplications are computationally more efficient than HTM multiplications~\cite{figueredo2013robust}. Third, the utilization of unit dual quaternions is comparatively simpler than the use of Cartesian coordinates and unit quaternions. This can be attributed to the fact that a series of rigid movements can be expressed as a sequence of dual quaternion multiplications, whereas in the other case, the computation of rotation and position is performed independently (for further details, refer to~\cite{figueredo2013robust,filipe2013rigid}). Fourth, unit dual quaternions yield two closed-loop equilibrium points associated to a quaternion's double coverage of $\mathrm{SO}(3)$, both of which represent the identity rotation matrix, while rotation matrices generate a minimum of four closed-loop equilibrium points, with only one of them relating to the identity~\cite{filipe2015adaptive}.  Fifth, in contrast to alternative techniques within $\mathrm{SE}(3)$, such as~\cite{lee2010geometric}, unit dual quaternions exclusively rely on a single error function, rather than the need for two separate functions addressing position and attitude errors. 

In light of these features, unit dual quaternions have been applied across a wide range of disciplines, including but not limited to inertial navigation~\cite{wu2005strapdown}, state estimation~\cite{zu2014distributed}, inverse kinematics~\cite{gan2008dual}, computer graphics~\cite{kavan2006dual} and computer vision~\cite{daniilidis1999hand}.

In the context of pose control, akin to~\cite{bullo1995proportional}, the authors in~\cite{han2008kinematic} and~\cite{wang2012geometric} broadened PD-alike feedback controllers to encompass the Lie group of unit dual quaternions using its logarithmic mapping. This resulted in a globally exponentially stable kinematic control law for pose regulation or tracking. These outcomes were subsequently expanded in~\cite{wang2013unit} to also account for rigid body dynamics. Since these findings, the unit dual quaternion-based pose-tracking problem has received considerable attention in the literature. To name a few, the need for linear and angular velocity feedback was dropped in~\cite{filipe2013rigid}, a backstepping control technique to account for robustness was proposed in~\cite{zhang2011robust}, adaptive control allowed for simultaneous pose-tracking and parameter identification in~\cite{filipe2015adaptive}, formation flying was addressed in~\cite{wang2012dual}, and optimal control variants in the form of linear quadratic regulator (LQR) and Model Predictive Control (MPC) were formulated in \cite{marinho2015dual} and \cite{lee2017constrained}, respectively.

Despite the achievements, all these methods exclusively focus on \emph{pose-tracking}, i.e., they track a time-varying position and attitude reference. However, not all problems fit in such a description. For a more intuitive understanding of this concept we use an illustrative example from~\cite{faulwasser2015nonlinear}: When aiming for precise steering of a robot tool along a geometric reference, the primary concern is to minimize the deviation between the reference and the tool, while the velocity to move along the reference is of secondary interest and can be modified to enhance accuracy. In other words, the problem is not centered on tracking a pre-defined time-varying reference, but rather on leveraging the velocity to traverse the reference as an additional degree of freedom. Such control problems are denoted as \emph{path-following}. Its appealing properties for a wide range of applications, accompanied by the fact that it is agnostic to the fundamental limitations of reference-tracking~\cite{aguiar2005path}, account for the significant attention path-following has received in literature. A detailed description of existing approaches can be found in~\cite{faulwasser2015nonlinear,hung2023review}. Among those, most of the existing path-following methods omit the rotational dynamics and focus on path convergence of the translational dynamics. 

To close this gap, we assume that the geometric reference consists of a desired path with a moving coordinate frame associated to it, and we define \emph{pose-following} as a generalization of path-following in which the goal is to steer the system's position and attitude\footnote{In contrast to~\cite{hung2023review, plaskonka2015different}, where the rotational convergence is reduced to the heading $\psi \in \mathbb{R}$ and the remaining two euler angles are left unattended, we seek \emph{true attitude convergence}, i.e. $q\in\mathrm{SO}(3)$.} along the reference. This begs the question of how to formulate such a pose-following method.

To answer this question, in this paper we derive a unit dual quaternion-based pose-following control approach for rigid body dynamics. To do so, we take advantage of the previously mentioned benefits of unit dual quaternions, namely singularity-free, compactness, computational efficiency and the logarithmic mapping associated with the Lie group, allowing us to extend the PD-alike feedback control from pose-tracking to pose-following. As a result, the freedom and versatility of path-following is augmented to full body motions, i.e., translations, as well as rotations. To the best of our knowledge, this is the first work that explicitly attempts to follow upon both the longitudinal and angular coordinates. 

More specifically, the presented scheme consists of the following contributions:
\begin{enumerate}
    \item We derive the equations of motion for the full pose error between the geometric reference and the rigid body in the form of a dual quaternion and dual twist.
    \item We extend the original control law to account for nonlinearities that arise from introducing auxiliary states associated with pose-following. Besides that, we design the additional degree of freedom either to ensure convergence to a desired velocity profile or as feedback. Therefore, the progress along the reference is versatile in the sense that it either accepts any velocity profile or can be utilized to incite a desired behavior around the geometric reference.
    \item We prove that the presented control law is almost globally asymptotically stable\footnote{"almost" globally asymptotically stability is the optimal performance achievable by a continuous controller for rotational motion, owing to the fact that the group of rotation matrices SO(3) constitutes a compact manifold~\cite{bhat2000topological}.}. When doing so, we take special care of the two equilibria problem, by introducing a switch that guarantees convergence to the closest equilibrium point. 
\end{enumerate}

The remainder of this paper is structured as follows: Section~\ref{sec:problem_statement} introduces the pose-following problem. Section~\ref{sec:solution_approach} presents the solution proposed in this paper, by revisiting the unit dual quaternion-based algebra, transforming the error dynamics into dual quaternion and dual twist form, deriving the control law and conducting a stability analysis. Experimental results are shown in Section~\ref{sec:experiments} before Section~\ref{sec:conclusions} presents the conclusions.

\vspace{2mm}
\noindent\textit{Notation:} We will use $\dot{(\cdot)} = \dv{(\cdot)}{t}$ for time derivatives and $\mathring{(\cdot)} = \dv{(\cdot)}{\theta}$ for differentiating over pose-parameter $\theta$. We denote three-dimensional vectors in bold $\bm{v}$, dual numbers as $a +\epsilon$b and dual quaternions as $\hat{q}$. We define $\hat{I}$ as $[1,0,0,0]+\epsilon[0,0,0,0]$.
\vspace{2mm}

\section{THE POSE-FOLLOWING PROBLEM}\label{sec:problem_statement}
\noindent Classical path-following approaches steer a system's position, while leaving the attitude unattended. In this work, we tighten the original path-following problem by focusing on full rigid body motions, i.e, position and attitude following.

\subsection{Rigid body dynamics}
\noindent The three-dimensional rigid body dynamics in the body frame are given by
\begin{subequations}\label{eq:rigid_body_dynamics}
\begin{gather}
    \Ddot{\bm{p}}^{b}(t)= \bm{f}^b(t)m^{-1}\,,\\
    \dot{\bm{\omega}}^{b}(t) = \text{J}^{-1} \left(\bm{\tau}^b(t) - \bm{\omega}^b(t) \times \text{J}\,\bm{\omega}^b(t) \right)\,,
\end{gather}
\end{subequations}
where $\{\bm{p}^b,\bm{\omega}^b, \bm{f}^b, \bm{\tau}^b\}\in\mathbb{R}^3$ refer to the rigid body's position, angular velocity, control forces and control torques, while $m\in\mathbb{R}$ and $J\in\mathbb{R}^{3\cross3}$ are the mass and inertia matrix. From now onward, since frame superscripts remain constant, they will be dropped. By taking position $\bm{p}$, longitudinal velocity $\dot{\bm{p}}$, attitude $q \in \mathrm{SO}(3)$ and angular velocity $\bm{\omega}$ as states $\bm{x}(t)=\left[\bm{p}(t), \dot{\bm{p}}(t), q(t), \bm{\omega}(t)\right]$ with forces $\bm{f}$ and torques $\bm{\tau}$ as inputs $\bm{u}(t)=\left[\bm{f}(t),\bm{\tau}(t)\right]$, and introducing the respective kinematic relationships, the dynamics in~\eqref{eq:rigid_body_dynamics} can be written in the standard form:
\begin{equation}\label{eq:rigid_body_f}
    \dot{\bm{x}}(t) = f(\bm{x}(t), \bm{u}(t))\,.
\end{equation}

\subsection{Geometric reference representation}
\noindent Let $\Gamma$ refer to a geometric reference and be defined as a path with a moving frame attached to it. Its respective desired position and attitude are given by two functions, $\bm{p}_d\,:\,\mathbb{R}\mapsto\mathbb{R}^3$ and $q_d\,:\,\mathbb{R}\mapsto\mathrm{SO}(3)$, that depend on pose-parameter $\theta$ and are at least $\mathcal{C}^2$:
\begin{equation}\label{eq:geom_ref}
    \Gamma = \{\theta \in[\theta_0,\theta_f] \subseteq\mathbb{R}\mapsto\bm{p}_d(\theta) \in \mathbb{R}^3, q_d(\theta) \in \mathrm{SO}(3)\}
\end{equation}
It should be noted that the $C^2$ requirement for $q_d$ enables the calculation of the desired angular velocity $\bm{\omega}_d(\theta)\,:\,\mathbb{R}\mapsto\mathbb{R}^3$ from its kinematic equations.

\subsection{Problem statement}
\noindent To incorporate the additional freedom inherited from path-following, we augment the rigid body dynamics in~\eqref{eq:rigid_body_f} by adding the pose-parameter $\theta(t)$ and its first time derivative $\dot{\theta}(t)$ as virtual states and assign the second time derivative $\ddot{\theta}(t)$ as a virtual input. The resulting system is denoted as
\begin{equation}\label{eq:rigid_body_f_aug}
    \dot{\bm{x}}_\Gamma(t) = f_\Gamma(\bm{x}_\Gamma(t), \bm{u}_\Gamma(t))\,,
\end{equation}
where $\bm{x}_\Gamma(t)=\left[\bm{p}(t), \dot{\bm{p}}(t), q(t), \bm{\omega}(t), \theta(t), \dot{\theta}(t)\right]$ and $\bm{u}_\Gamma(t)=\left[\bm{f}(t),\bm{\tau}(t), \ddot{\theta}(t)\right]$. The augmented system $f_\Gamma$ contains two additional equations of motion that relate to the integration chain of the pose-parameter $\theta(t)$, implying that the virtual input $\ddot{\theta}(t)$ is associated to its acceleration. Consequently, the time evolution $\theta(t)$, and thereby the pose reference $\{\bm{p}_d(\theta(t)),q_d(\theta(t))\}$, are controlled via the virtual input $\ddot{\theta}(t)$. This leads to the definition of the \emph{pose-following error} as
\begin{equation}
    \bm{e}_\Gamma(t) = \triangle \left[\{\bm{p}(t),q(t)\},\{\bm{p}_d(\theta(t)),q_d(\theta(t))\}\right]\,,
\end{equation}
where $\triangle\,:\,\{\mathbb{R}^3,\,\mathrm{SO}(3)\}\cross\{\mathbb{R}^3,\,\mathrm{SO}(3)\}\mapsto\mathbb{R}$ is a function that outputs the deviation between the rigid body's pose and reference pose, and will only be $0$ if both are equal, i.e., $\triangle\left[a,b\right] = 0 \iff a=b$. Due to the structure of $\mathrm{SE}(3)$, this function is dependent on the control design approach, and thus, will be defined in the upcoming Section~\ref{sec:solution_approach}.
For the remainder of this work, we address the following problem:
\begin{problem}[\bfseries Pose-Following]
Given the geometric reference $\Gamma$ in~\eqref{eq:geom_ref} and the augmented rigid body dynamics in~\eqref{eq:rigid_body_f_aug}, formulate a controller $\bm{u}_\Gamma(t) = \left[\bm{f}(t),\bm{\tau}(t),\ddot{\theta}(t)\right]$ that fulfills:
\begin{itemize}
    \item[P1.1] \textbf{Pose convergence:} The pose-following error vanishes asymptotically $\lim_{t\to\infty} \bm{e}_\Gamma(t) = 0\,$.
    \item[P1.2] \textbf{Convergence on pose-parameter:} The system converges to the end of the geometric reference $\lim_{t\to\infty} \theta_f-\theta(t) = 0$.
\end{itemize}
\end{problem}
\noindent For specific applications, it might be of interest to traverse the reference according to a desired velocity profile $\theta_{vd}(\theta(t))$. For example, inspection and manufacturing processes might require a lower traverse velocity at critical sections of the geometric reference. For this reason, in contrast to~\cite{faulwasser2015nonlinear}, instead of directly depending on time, our velocity profile $\theta_{vd}$ depends on the pose-parameter $\theta(t)$. Notice that this problem differs from reference tracking, for further details refer to Sec II. in~\cite{faulwasser2015nonlinear}. Remaining consistent with the existing literature, we denote this problem as \emph{pose-following with velocity assignment}.
\begin{problem}[\bfseries Pose-Following with velocity assignment]
Given the geometric reference $\Gamma$ in~\eqref{eq:geom_ref} and the augmented rigid body dynamics in~\eqref{eq:rigid_body_f_aug}, formulate a controller $\bm{u}_\Gamma(t) = \left[\bm{f}(t),\bm{\tau}(t),\ddot{\theta}(t)\right]$ that fulfills:
\begin{itemize}
    \item[P2.1] \textbf{Pose convergence:} P1.1 from Problem 1.
    \item[P2.2] \textbf{Velocity convergence:} The velocity of the pose-parameter converges to a desired velocity profile $\lim_{t\to\infty} \dot{\theta}(t)-\theta_{vd}(\theta(t)) = 0$.
\end{itemize}
\end{problem}
\section{SOLUTION APPROACH}\label{sec:solution_approach}
\subsection{Mathematical preliminaries}
\noindent For the sake of making this paper self-contained we briefly recall the dual quaternion algebra. In doing so, we follow the notation and content of~\cite{wang2013unit}, which pioneered the use of unit dual quaternion-based pose-tracking for rigid body dynamics. To begin, we define the foundational concepts of the quaternion and the dual number, which serve as the building blocks of the dual quaternion. For more comprehensive information on these concepts, we refer the reader to~\cite{yang1963application,wang2012geometric,kenwright2012beginners}.

\subsubsection{Quaternions:} Quaternions extend the notion of a complex number to the four-dimensional space $\mathbb{R}^4$ and can be expressed as $q = a + b\,i + c\,j + d\,k = \left[s,\bm{v}\right]$, where $s$ is the \emph{scalar part}, $\bm{v}\in\mathbb{R}^3$ is the \emph{vector part} and $\{i,j,k\}$ is the standard $\mathbb{R}^4$ basis, i.e., $i^2 = j^2 = k^2 = -1$ and $ij = k,\,jk=i,\,ki=j$. When a three-dimensional vector is expressed as a quaternion with a zero scalar component, the term \emph{vector quaternion} is used. Furthermore, quaternions fulfilling $a^2 + b^2 + c^2 + d^2 = 1$ are named \emph{unit quaternions} and allow to describe rotations, i.e., a rotation around a unit axis $\bm{n}$ by an angle of $|\phi|<2\pi$ can be expressed as a unit quaternion in the form of $q = \left[\cos(\phi/2),\sin(\phi/2)\bm{n}\right]$. Additionally, unit quaternions constitute a Lie group $\mathcal{Q}_u$ over multiplication, and its logarithmic map is $\ln \hat{q} = \phi/2$ with $\phi=\left[0,\,|\phi|\bm{n}\right]$. Limiting to $\phi\in[0,\,2\pi)$ and defining $v$ as the Lie algebra of $\mathcal{Q}_u$, i.e., all logarithmic mappings of unit quaternions, the adjoint transformation is $\text{Ad}_q V = q \circ V \circ q^-1 = q \circ V \circ q^*$, where '$\circ$' is the quaternion multiplication and $V\in v$. 
\subsubsection{Dual numbers and vectors:} Dual numbers are defined as $\hat{a} = a + \epsilon b$ with $\epsilon^2=0$, $\epsilon\neq 0$ -- $\epsilon$ is nilpotent --, and $\{a,b\}\in\mathbb{R}$. In this case $a$ and $b$ are denoted as the \emph{real part} and \emph{dual part}, respectively. Dual vectors are a generalization of dual numbers, where both the real and dual part are three-dimensional vectors such that $\hat{\bm{a}} = \bm{a} + \epsilon\bm{b}$, where $\{\bm{a},\bm{b}\}\in\mathbb{R}^3$.

\subsubsection{Dual quaternions:} A dual quaternion features dual components instead of its regular ones, i.e., $\hat{q} = \left[\hat{s},\bm{\hat{v}}\right]$, where $\hat{s}$ is a dual number and $\bm{\hat{v}}$ is a dual vector. Following the introduced notation, a dual quaternion with a vanishing scalar part is named as a \emph{dual vector quaternion}. We define the dot product of two dual vector quaternions $\hat{v}=\left[\hat{0},\hat{\bm{v}}\right]$ with $\hat{\bm{v}} = \bm{v}_r + \epsilon\bm{v}_d$ and  $\hat{k}=[\hat{0},\hat{\bm{k}}]$ with $\hat{\bm{k}} = \bm{k}_r + \epsilon\bm{k}_d = \left[k_{r1},k_{r2},k_{r3}\right] + \epsilon\left[k_{d1},k_{d2},k_{d3}\right]$ as 
\begin{equation*}
\hat{k}\odot\hat{v} = \left[\hat{0},K_r\bm{v}_r\right] + \epsilon\left[\hat{0},K_d\bm{v}_d\right]
\end{equation*}
with $K_r = \text{diag}(k_{r1},k_{r2},k_{r3})$ and $K_d = \text{diag}(k_{d1},k_{d2},k_{d3})$. Alternatively, a dual quaternion can also be expressed as $\hat{q} = q_d + \epsilon q_r$, where $q_d$ and $q_r$ are quaternions. To operate with dual quaternions, we introduce the following operations:
\begin{align*}
    &\hat{q}_1 + \hat{q}_2 = \left[\hat{s}_1 + \hat{s}_2, \hat{\bm{v}}_1 + \hat{\bm{v}}_2\right] = \left(q_{r1} + q_{r2}\right) + \epsilon\left(q_{d1} + q_{d2}\right)\,,\\
    &\lambda \hat{q} = \left[\lambda\hat{s}, \lambda\hat{\bm{v}}\right] = \lambda q_{r} + \epsilon\lambda q_{d}\,,\\
    \begin{split}
    &\hat{q}_1 \circ \hat{q}_2 = \left[\hat{s}_1\hat{s}_2 -\hat{\bm{v}}_1^T\odot\hat{\bm{v}}_2, \hat{s}_1\hat{\bm{v}}_2 + \hat{s}_2\hat{\bm{v}}_1 + \hat{\bm{v}}_1\cross\hat{\bm{v}}_2\right]  = \,\\
    &\hspace{1cm}=q_{r1}\circ q_{r_2} + \epsilon\left(q_{r1}\circ q_{d2} + q_{d1}\circ q_{r2}\right)\,,
    \end{split}
\end{align*}
where $\hat{q}_1$ and $\hat{q}_2$ are dual quaternions, $\lambda\in\mathbb{R}$ and the operator '$\circ$' is the (dual) quaternion multiplication, which is associative and distributive but not commutative. Other relevant properties of the dual quaternion are the conjugate $\hat{q}^* = \left[\hat{s},-\hat{\bm{v}}\right]$ and the multiplicative inverse $\hat{q}^{-1} = \left(1/\hat{q}\circ\hat{q}^*\right)\circ \hat{q}^*$. It follows that dual quaternions that meet the condition $\hat{q}\circ\hat{q}^*=\hat{I}$ also satisfy $\hat{q}^{-1} = \hat{q}^*$. In such instances, the dual quaternion in question is referred to as a \emph{unit dual quaternion}.

Just as unit quaternions permit the representation of rotations, unit dual quaternions provide a means of describing three-dimensional transformations that encompass both translation and rotation. Specifically, a transformation consisting of a translation vector $\bm{p}$ and a rotation quaternion $q$ corresponds to a screw motion, which entails a translation along axis $\bm{n}$ by a distance $d$, and a rotation of angle $|\phi|$. Such a transformation is expressed as a dual quaternion in the following form:
\begin{equation}\label{eq:dq}
    \hat{q} = \left[\cos\left(\hat{\phi}/{2}\right),\,\sin\left(\hat{\phi}/{2}\right)\hat{\bm{n}}\right] = q + \epsilon/{2}\,p\circ q\,,
\end{equation}
where $\hat{\bm{n}}$ is the dual screw axis and $\hat{\phi} = |\phi| + \epsilon d$ is the dual screw angle. Similar to unit quaternions, unit dual quaternions form a Lie group $\mathcal{DQ}_u$ with respect to the dual quaternion multiplication and its logarithmic map is also a dual quaternion given by
\begin{equation}
    \ln \hat{q} = 1/2\left(\phi+\epsilon\,p\right)\,,
\end{equation}
where $\phi = \left[0,\,|\phi|\,\bm{n}\right]$. In a similar way to unit quaternions, naming $\hat{v}$ as the Lie algebra for $\mathcal{DQ}_u$, the adjoint transformation for the dual quaternions is 
$\text{Ad}_{\hat{q}} \hat{V} = \hat{q} \circ \hat{V} \circ \hat{q}^-1 = \hat{q} \circ \hat{V} \circ \hat{q}^*$, where $\hat{V}\in\hat{v}$. For additional information regarding the Lie group of unit dual quaternions, please refer to~\cite{wang2012geometric}.


\subsection{Unit dual quaternion dynamics}\label{sec:dq_dynamics}
\noindent In this subsection we aim to transform the rigid body dynamics in~\eqref{eq:rigid_body_f} to a unit dual quaternions representation. For this purpose, in a similar way to~\cite{wang2013unit}, we start by derivating~\eqref{eq:dq} in time, which results in the following kinematic equations:
\begin{subequations}\label{eq:dq_kinem}
    \begin{align}
       \dot{\hat{q}}(t) &= \frac{1}{2}\hat{\omega}(t)\circ\hat{q}(t)  \,,\label{eq:dq_dot}\\
        \hat{\omega}(t) &= \left[0,\bm{\omega}(t)\right] + \epsilon\left[0,\dot{\bm{p}}(t) + \bm{p}(t)\cross\bm{\omega}(t)\right]\label{eq:dt}
    \end{align}
where $\hat{\omega}(t)$ is the dual twist. Taking its time derivative, we get
\begin{equation}
       \dot{\hat{\omega}}(t) = \dot{\bm{\omega}}(t) + \epsilon\left(\Ddot{\bm{p}}(t) + \dot{\bm{p}}(t)\cross\bm{\omega}(t) + \bm{p}(t)\cross\dot{\bm{\omega}}(t)\right)\,,\label{eq:dt_kinem}
\end{equation}
and combining it with the rigid body dynamics in~\eqref{eq:rigid_body_dynamics}, leads to
\begin{flalign}
    &\dot{\hat{\omega}}(t) = \left(\bm{a}+\text{J}^{-1}\bm{\tau}\right) +\epsilon\left(\bm{f}/m+\dot{\bm{p}}\cross\bm{\omega}+\bm{p}\cross\left(\bm{a}+\text{J}^{-1}\bm{\tau}\right)\right) \notag\\
    &= \underbrace{\bm{a} + \epsilon\left(\bm{p}\cross\bm{a}+\dot{\bm{p}}\cross\bm{\omega}\right)}_{\hat{F}} + \underbrace{\text{J}^{-1}\bm{\tau}+\epsilon\left(\bm{f}/m +\bm{p}\cross\text{J}^{-1}\bm{\tau}\right)}_{\hat{U}}\notag\\
    &= \hat{F}(t) + \hat{U}(t) \label{eq:dt_FU}
\end{flalign}
with $\bm{a} =  \text{J}^{-1} \bm{\omega} \times \text{J}\,\bm{\omega}\,$. For readability, in the first two lines of~\eqref{eq:dt_FU} dependencies on time $(\cdot)(t)$ have been omitted. Notice that $\hat{F}(t)$ is fully defined by the rigid body's state $\bm{x}(t)$, while the force and torque control inputs $\bm{u}(t)$ only appear in $\hat{U}(t)$.
\end{subequations}
\begin{model}[\textbf{Unit dual quaternion-based dynamics}]
Letting eq.~\eqref{eq:dq} describe a screw motion for translation $\bm{p}(t)$ and rotation $q(t)$, and defining the linear and angular velocity as in eq.~\eqref{eq:rigid_body_f}, then the unit dual quaternion-based rigid body dynamics are 
\begin{subequations}\label{eq:udq}
    \begin{align}
        \dot{\hat{q}}(t) &= \frac{1}{2}\hat{\omega}(t)\circ\hat{q}(t)  \,,\\
        \hat{\omega}(t) &= \bm{\omega}(t) + \epsilon\left(\dot{\bm{p}}(t) + \bm{p}(t)\cross\bm{\omega}(t)\right)\,,\\
        \dot{\hat{\omega}}(t)&= \hat{F}(t) + \hat{U}(t)
    \end{align}
\end{subequations}
where $\hat{F}(t)$ and $\hat{U}(t)$ are given in~\eqref{eq:dt_FU}.
\end{model}

\subsection{Unit dual quaternion error dynamics}
\noindent Drawing upon the dynamics obtained in the previous subsection, we proceed to derive the unit dual quaternion-based error dynamics for the pose-following problem. By applying the dual quaternion and twist definitions in eqs.~\eqref{eq:dq} and~\eqref{eq:dt}, it is possible to transform the geometric reference $\Gamma$ in~\eqref{eq:geom_ref} into a desired dual quaternion and a desired dual twist: 
\begin{subequations}\label{eq:pose_des}
    \begin{align}
       \hat{q}_d(\theta) &= q_d(\theta) + \epsilon/2\,p_d(\theta)\circ q_d(\theta) \,,\label{eq:dq_th}\\
        \hat{\omega}_d(\theta) &= \left[0,\bm{\omega}_d(\theta)\right] + \epsilon\left[0,\mathring{\bm{p}}_d(\theta) + \bm{p}_d(\theta)\cross\bm{\omega}_d(\theta)\right]\,,\label{eq:dt_ph}
    \end{align}
\end{subequations}
where $\theta\in[\theta_0,\theta_f]$. Combining~\eqref{eq:pose_des} with the kinematics in~\eqref{eq:dq_kinem}, the equations of motion for the desired pose are obtained:
\begin{subequations}\label{eq:pose_des_ode}
    \begin{align}
       &\mathring{\hat{q}}_d(\theta) = \frac{1}{2}\hat{\omega}_d(\theta)\circ\hat{q}_d(\theta)  \,,\label{eq:dqd_kinem1}\\
       \begin{split}
       &\mathring{\hat{\omega}}_d(\theta) = \left[0,\mathring{\bm{\omega}}_d(\theta)\right] +\\
       &\hspace{0.4cm}\epsilon\left[0,\ringring{\bm{p}}_d(\theta) + \mathring{\bm{p}}_d(\theta)\cross\bm{\omega}_d(\theta) + \bm{p}_d(\theta)\cross\mathring{\bm{\omega}}_d(\theta)
       \right]\,,\label{eq:dqd_kinem2}
       \end{split}
    \end{align}
\end{subequations}
From~\eqref{eq:pose_des_ode} it is apparent that in contrast to the pose-tracking case~\cite{wang2013unit}, the desired pose in~\eqref{eq:pose_des} does not evolve according to \emph{time} $t$, but with respect to the \emph{pose-parameter} $\theta$. The error between the geometric reference and the rigid body's pose can be expressed in the form of a unit dual quaternion:
\begin{equation}\label{eq:dq_error}
\hat{q}_e(t) = \hat{q}(t)\circ \hat{q}_d^*(\theta(t))\,,
\end{equation}
Derivating the dual quaternion error~\eqref{eq:dq_error} in time\footnote{Time derivations over pose-parameter $\theta$ dependent variables, such as the $\hat{q}_d(\theta(t))$ requires using the chain rule, i.e., $\dv{(\cdot)}{t} = \dv{(\cdot)}{\theta}\dv{\theta}{t} = \mathring{(\cdot)}\dot{\theta}(t)$} leads to
\begin{equation*}
    \dot{\hat{q}}_e(t) = \dot{\hat{q}}(t)\circ\hat{q}_d^*(\theta(t)) + \dot{\theta}(t)\hat{q}(t)\circ\mathring{\hat{q}}_d^*(t)\,.
\end{equation*}
Combining it with~\eqref{eq:dq_dot}, ~\eqref{eq:dqd_kinem1},~\eqref{eq:dq_error} and the property $(\hat{q}_1\circ\hat{q}_2)^* = \hat{q}_2^*\circ\hat{q}_1^* $ results in
\begin{equation*}
\dot{\hat{q}}_e(t) = \frac{1}{2}\left(\hat{\omega}(t)\circ\hat{q}_e(t) + \dot{\theta}(t)\hat{q}_e(t)\circ \hat{\omega}_d^*(\theta(t))\right)\,.
\end{equation*}
Noticing that $\hat{q}_e \circ \hat{\omega}_d^* = \left(\hat{q}_e\circ\hat{\omega}_d^*\circ\hat{q}_e^*\right)\circ\hat{q}_e $, the equation above can be rearranged to
\begin{equation*}
    \dot{\hat{q}}_e(t) = \frac{1}{2}\left(\hat{\omega}(t) + \dot{\theta}(t)\hat{q}_e(t)\circ\hat{\omega}_d^*(\theta(t))\circ\hat{q}_e^*(t)\right) \circ\, \hat{q}_e(t)\,,
\end{equation*}
which takes the same form as~\eqref{eq:dq_dot}: 
\begin{subequations}
    \begin{align}
    \dot{\hat{q}}_e(t) &= \frac{1}{2}\hat{\omega}_e(t)\circ\, \hat{q}_e(t)\,,\label{eq:dqe_dot}\\
    \hat{\omega}_e(t) &= \hat{\omega}(t) + \dot{\theta}(t)\text{Ad}_{\hat{q}_e(t)}\hat{\omega}_d^*(\theta(t))\,.\label{eq:dte}
    \end{align}
\end{subequations}
When compared to the pose-tracking case, the first time derivative of the pose-parameter $\dot{\theta}(t)$ appears to be multiplying the second term of the dual twist error. Similar derivations result in $p_e(t) = p(t) + \text{Ad}_{q_e(t)}p_d^*(\theta(t))$ and $w_e(t) = w(t) + \dot{\theta}(t)\text{Ad}_{q_e(t)}\omega_d^*(\theta(t))$. These expressions allow to ensure that the right-hand side of~\eqref{eq:dte} is equivalent to
\begin{equation}\label{eq:we}
    \hat{\omega}_e(t) = \left[0,\bm{\omega}_e(t)\right] + \epsilon\left[0, \dot{\bm{p}}_e(t) + \bm{p}_e(t)\cross\bm{\omega}_e(t)\right]\,.
\end{equation}
\ifcomment
For brevity, we omit these derivations. In case of interest, a detailed description of the respective procedure can be found in the Appendix~\ref{apend:dt_e_app}. 
\fi
Other than that, to fully define the error dynamics, we still need to compute the time derivative of the dual twist error. Towards this end, we take the time derivative of~\eqref{eq:dte} and we obtain
\begin{align*}
    \begin{split}
    \dot{\hat{\omega}}_e(t) =\,\dot{\hat{\omega}}(t) &+ \ddot{\theta}(t)\,\text{Ad}_{\hat{q}_e(t)}\hat{\omega}_d^*(\theta(t)) +\\
    \dot{\theta}(t)&\left[ \dot{\hat{q}}_e(t)\circ\hat{\omega}_d^*(\theta(t))\circ\hat{q}_e(t)+ \right.\\
     &\hspace{0.25cm}\hat{q}_e(t)\circ\dot{\theta}(t)\mathring{\hat{\omega}}_d^*(\theta(t))\circ\hat{q}_e(t)+\\
     &\hspace{0.2cm}\left.\hat{q}_e(t)\circ\hat{\omega}_d^*(\theta(t))\circ\dot{\hat{q}}_e(t) \right]\,.
    \end{split}
\end{align*}
\begin{model}[\textbf{Unit dual quaternion-based error dynamics}]
For a given dual quaternion state $\hat{q}(t)$ and a desired configuration $\hat{q}_d(\theta(t))$ -- associated to pose-parameter $\theta(t)$ --, the dynamics of the dual quaternion error in~\eqref{eq:dq_error} are
\begin{subequations}\label{eq:udq_error}
    \begin{align}
        \dot{\hat{q}}_e(t) &= \frac{1}{2}\hat{\omega}_e(t)\circ\hat{q}_e(t)  \,,\\
        \hat{\omega}_e(t) &= \left[0,\bm{\omega}_e(t)\right] + \epsilon\left[0, \dot{\bm{p}}_e(t) + \bm{p}_e(t)\cross\bm{\omega}_e(t)\right]\,,\\
        \dot{\hat{\omega}}_e(t)&= \hat{F}(t) + \hat{U}(t) + \ddot{\theta}(t)\,\text{Ad}_{\hat{q}_e(t)}\hat{\omega}_d^*(\theta(t)) +\notag\\
        \begin{split}
        &\hspace{0.5cm}\dot{\theta}(t)\left[ \dot{\hat{q}}_e(t)\circ\hat{\omega}_d^*(\theta(t))\circ\hat{q}_e(t)+ \right.\\
        &\hspace{1.3cm}\hat{q}_e(t)\circ\dot{\theta}(t)\mathring{\hat{\omega}}_d^*(\theta(t))\circ\hat{q}_e(t)+\\
         &\hspace{1.2cm}\left.\hat{q}_e(t)\circ\hat{\omega}_d^*(\theta(t))\circ\dot{\hat{q}}_e(t) \right]\,\label{eq:udq_e3},
        \end{split}
    \end{align}
\end{subequations}
where $\hat{q}_e(t) = \hat{q}\circ\hat{q}_d^*(\theta(t))$, $p_e(t) = p(t) + \text{Ad}_{q_e(t)}p_d^*(\theta(t))$, $w_e(t) = w(t) + \dot{\theta}(t)\text{Ad}_{q_e(t)}\omega_d^*(\theta(t))$, and $\hat{F}(t)$ and $\hat{U}(t)$ are given in~\eqref{eq:dt_FU}.
\end{model}
\noindent The first two equations, i.e., the time derivative of the dual quaternion error and the dual twist error, show the same structure as in pose-tracking. However, differences arise in the time derivative of the dual twist error, as it involves additional terms multiplied by the first and second time derivatives of the pose-parameter $\theta(t)$.
\subsection{Control law}
\noindent Considering the error dynamics in~\eqref{eq:udq_error}, the \emph{pose-convergence} definition in P1.1 can be reformulated as  $\lim_{t\to\infty} \hat{q}_e(t) = \pm\hat{I}\,$ and $\lim_{t\to\infty} \hat{\omega}_e(t) = \hat{0}\,$. To accomplish this, in this subsection we derive a control law for the term $\hat{U}$ in equation~\eqref{eq:udq_e3}. Its design is significantly influenced by the forthcoming stability analysis. Once the control law is determined, we will be able calculate the command forces $\bm{f}$ and torques $\bm{\tau}$ from eq.~\eqref{eq:dt_FU}. In a similar way to~\cite{bullo1995proportional} and~\cite{wang2013unit}, we decouple it into a feedforward (FF) and a feedback term (FB). The former eliminates nonlinearities in~\eqref{eq:udq_e3} and the latter ensures stability.
\begin{equation}\label{eq:U}
 \hat{U} = \hat{U}_{\text{FF}} + \hat{U}_{\text{FB}}   
\end{equation}

\begin{figure*}[t]
\centering
\includegraphics[width=\textwidth]{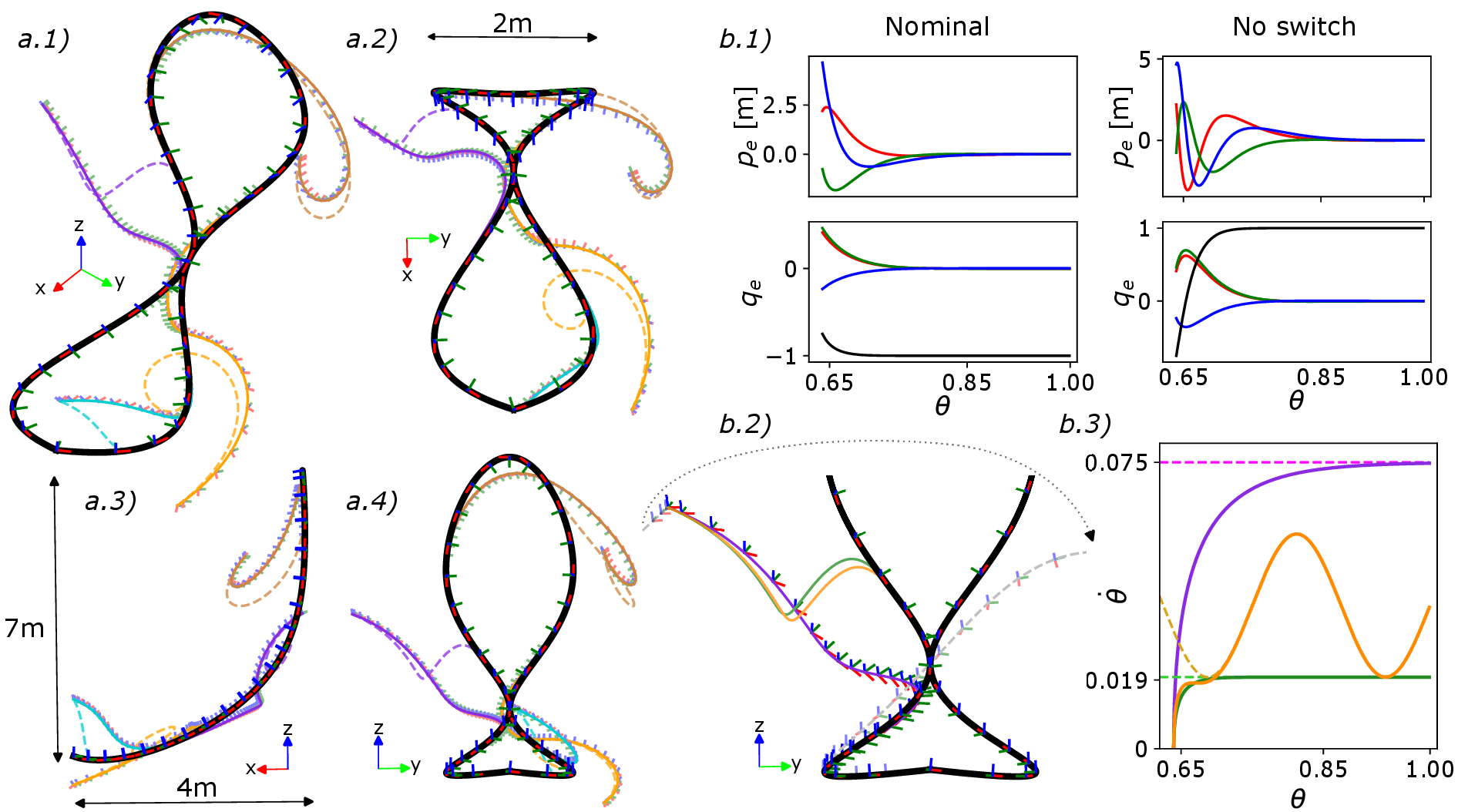}
\caption{Rigid body motions obtained from applying control law~\eqref{eq:U} -- with $\hat{U}_\theta$ defined as in Theorem 2 -- to the dynamics in~\eqref{eq:rigid_body_f}. The left column (a.1 - a.4) shows the almost global asymptotic stability of the presented pose-following control law by starting from different initial poses (purple, orange, yellow, cyan) from different perspectives. Each starting point is evaluated according to two different constant velocity profiles. The motions associated to the fast profiles are depicted by a continuous line, as well as their orientations, while the motions related to the slow profiles are given by dashed lines. The plots at the top right (b.1), together with the gray dashed line in (b.2), showcase the consequences of deactivating the $\lambda$ switch. The figure (b.3) shows that the pose-parameter's velocity (continuous line) converges to the desired velocity profile (dashed line) by evaluating three different cases (fast in purple, slow in green and sinusoidal in orange).}\label{fig:gas_results}
\vspace{-7mm}
\end{figure*}

From ~\eqref{eq:udq_e3} it is apparent that the first and the last term can readily be cancelled out by the feedforward compensation. However, this does not hold true for the second adjoint term, which is multiplied by the virtual input $\ddot{\theta}(t)$. To account for this, we choose $\ddot{\theta}(t) = U_\theta (x_\Gamma(t))$, where $U_\theta(\cdot)$ is the -- yet to be defined -- \emph{pose-parameter control law} dependent on the augmented state vector $x_\Gamma(t)$ in~\eqref{eq:rigid_body_f_aug}. This choice allows us to also include the adjoint term into the feedforward\footnote{Model 1 in~\eqref{eq:udq} and Model 2 in~\eqref{eq:udq_error} enable the conversion of $x_\Gamma(t)$ into a dual quaternion error and a dual twist error, as well as the conversion of these errors back into $x_\Gamma(t)$.}:
\begin{align}
    \hat{U}_\text{FF}&(t, U_\theta) = -\hat{F}(t) -U_\theta(x_\Gamma(t))\,\text{Ad}_{\hat{q}_e(t)}\hat{\omega}_d^*(\theta(t))-\notag\\
    &\dot{\theta}(t)\left[ \dot{\hat{q}}_e(t)\circ\hat{\omega}_d^*(\theta(t))\circ\hat{q}_e(t)+ \right.\label{eq:U_ff}\\
     &\hspace{-0.05cm}\hat{q}_e(t)\circ\dot{\theta}(t)\mathring{\hat{\omega}}_d^*(\theta(t))\circ\hat{q}_e(t)+
    \left.\hat{q}_e(t)\circ\hat{\omega}_d^*(\theta(t))\circ\dot{\hat{q}}_e(t) \right]\notag\,.
\end{align}
Regarding the feedback term, following the pose-tracking formulation in~\cite{wang2013unit}, we leverage the logarithmic mapping associated to the Lie group of unit dual quaternions $\mathcal{DQ}_u$ to design a proportional derivative feedback as
\begin{equation}\label{eq:u_fb}
    \hat{U}_\text{FB}(t) = -2\hat{\bm{k}}_p\odot\ln\lambda\hat{q}_e(t) - \hat{\bm{k}}_v\odot\hat{\omega}_e(t)\,,
\end{equation}
where $\hat{\bm{k}}_p$ and $\hat{\bm{k}}_v$ are vector dual quaternion control gains and $\lambda\in\{-1,1\}$ is a switching parameter to account for both equilibrium points $\pm\hat{I}$. This is defined as $\lambda = 1$, if $\hat{q}_{e_1}(t)>=0$ and $-1$ otherwise, where $\hat{q}_{e_1}$ refers to the first component of $\hat{q}_e(t)$.

\subsection{Stability analysis}
\noindent In the present subsection, we establish the necessary conditions for control laws $\hat{U}$ and $U_\theta$ to almost global asymptotic stability\footnote{We assume perfect and instantaneous state measurements, and thus, the presented global attractivity might be jeopardized by noises and delays that arise in practical applications. This can formally be addressed by combining the proposed method with robust control.}.
\begin{theorem}[\textbf{Stability of pose-following}]
Consider the geometric reference~\eqref{eq:geom_ref}, the augmented system~\eqref{eq:rigid_body_f_aug}, the control law $\hat{U}$ in~\eqref{eq:U} with the feedforward and feedback terms in~\eqref{eq:U_ff} and~\eqref{eq:u_fb}, and suppose that the following conditions are satisfied:
\begin{itemize}
    \item [i] The dual quaternion control gains are chosen as $\hat{\bm{k}}_p > \hat{0}$ with $k_{pd1} = k_{pd2}= k_{pd3}$, i.e., equivalent terms in the dual part of $\hat{\bm{k}}_p$, and $\hat{\bm{k}}_v > \hat{0}$.
    \item [ii] The pose-parameter control law ensures that the velocity of the pose-parameter is positive, i.e., $U_\theta(x_\Gamma(t))\implies\dot{\theta}(t) > 0,\;\forall\,\theta\in[\theta_0,\theta_f]$.
\end{itemize}
Then, the closed-loop control scheme defined by system~\eqref{eq:rigid_body_f} and control law~\eqref{eq:U} solves the pose-following Problem 1.
\end{theorem}

\begin{proof} 
Starting with \emph{pose convergence} in P1.1, since the feedforward term~\eqref{eq:U_ff} was designed to eliminate  all the nonlinearities in~\eqref{eq:udq_e3}, substituting~\eqref{eq:U} in~\eqref{eq:udq_e3} results in $\dot{\hat{\omega}}_e = \hat{U}_{\text{FB}}$. In addition, considering that~\eqref{eq:we} also remains true for pose-following, the stability analysis in~\cite{wang2013unit} holds. This implies that Model 2 converges to the closest equilibrium point $\{\pm\hat{I},\hat{0}\}$ asymptotically, which directly translates to the fulfillment of pose convergence.

Regarding \emph{convergence on pose-parameter} in P1.2, combining the Lyapunov function $V = ||\theta(t) - \theta_f||^2$ with $\dot{\theta}(t)>0$ -- from (ii) -- shows that $\theta_f$ is an asymptotically stable equilibrium point.
\end{proof}
\begin{theorem}[\textbf{Stability of pose-following with velocity assignment}]
Consider the geometric reference~\eqref{eq:geom_ref}, the augmented system~\eqref{eq:rigid_body_f_aug}, the control law $\hat{U}$ in~\eqref{eq:U} with the feedforward and feedback terms in~\eqref{eq:U_ff} and~\eqref{eq:u_fb}, and suppose that the following conditions are satisfied:
\begin{itemize}
    \item [i] The dual quaternion control gains are chosen as $\hat{\bm{k}}_p > \hat{0}$ with $k_{pd1} =k_{pd2}=k_{pd3}$ and $\hat{\bm{k}}_v > \hat{0}$.
    \item [ii] The pose-parameter control law is given by
    $U_\theta(x_{\Gamma}(t)) = -k_\theta\left(\dot{\theta}(t) - \theta_{vd}(\theta(t))\right) + \dot{\theta}(t)\mathring{\theta}_{vd}(\theta(t))$, where $k_\theta\in\mathbb{R}_{>0}$.
\end{itemize}
Then, the closed-loop control scheme defined by system~\eqref{eq:rigid_body_f} and control law~\eqref{eq:U} solves the pose-following with velocity assignment Problem 2.
\end{theorem}

\begin{proof}
The proof for \emph{pose convergence} in P2.1 remains the same as P1.1 in Theorem 1. When it comes to \emph{velocity convergence} in P2.2, the utilization of the Lyapunov function $V=||\dot{\theta}(t)-\theta_{vd}(\theta(t))||^2$ in conjunction with $U_\theta(x_{\Gamma}(t))$ as given in (iii), and the recognition that $\ddot{\theta}(t) = U_\theta(x_{\Gamma}(t))$, indicates that the velocity of the pose-parameter asymptotically converges to the desired velocity profile.
\end{proof}

\begin{figure*}[t]
\centering
\includegraphics[width=\textwidth]{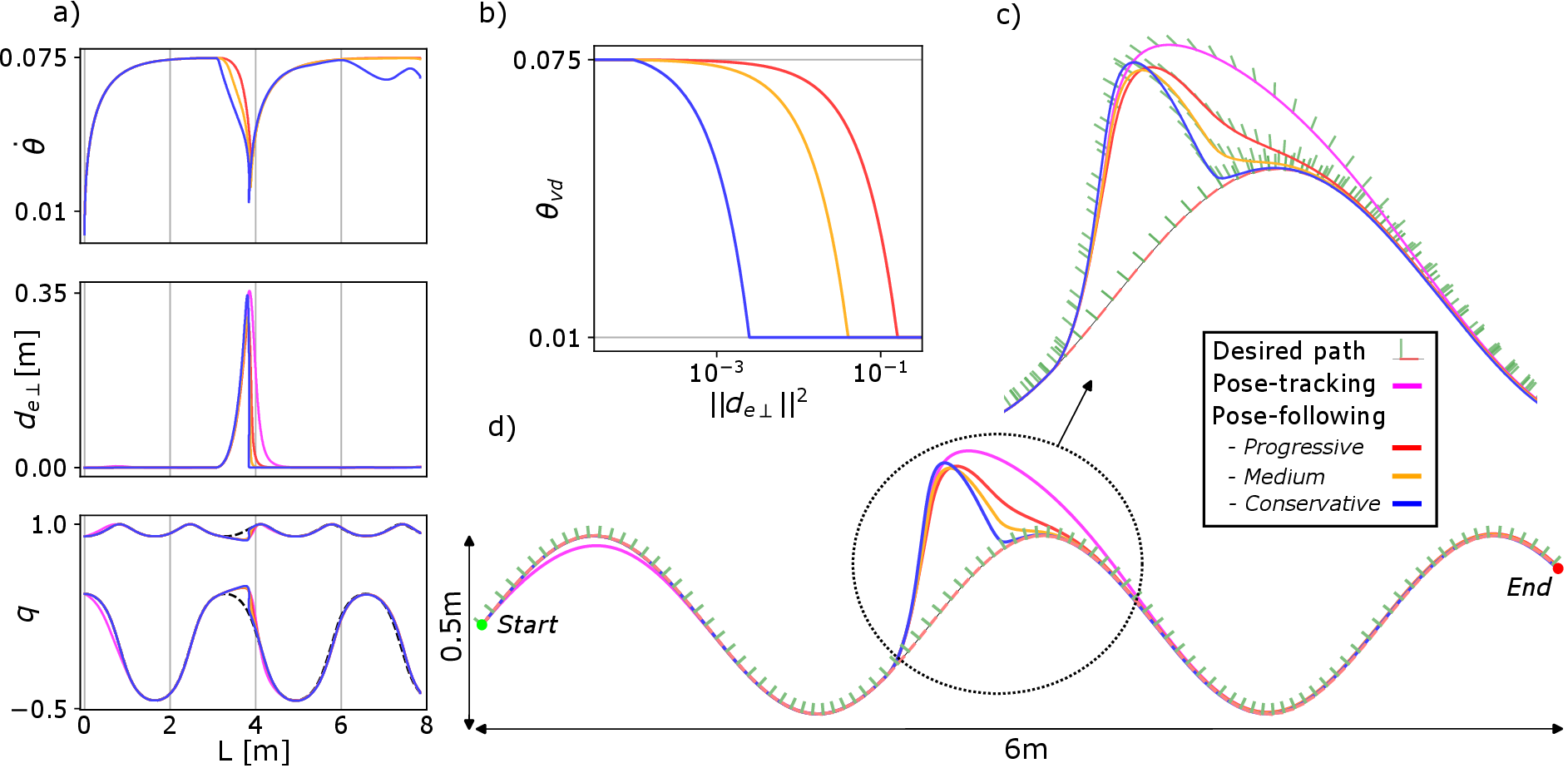}
\caption{A comparison between pose-tracking~\cite{wang2013unit} (magenta) and three variants of the presented pose-following control approach (progressive as red, medium as orange and conservative as blue) in the presence of a disturbance. The geometric reference is given by a thin black line and its moving frame. For the rest of the motions, only the normal components are shown for clarity. a) From top to bottom, velocity of pose-parameter, transverse distance to the geometric reference and first and last components of the unit quaternion, where the dashed black line refers to the geometric reference's orientation. b) The desired velocity profile as a mapping from transverse distance to the reference. c) Zoomed comparison in the location of the disturbance and d) Overview of the case-study.}\label{fig:disturbance_results}
\vspace{-7mm}
\end{figure*}
\section{NUMERICAL EXPERIMENTS}\label{sec:experiments}
\noindent In order to assess the effectiveness of our methodology, we focus on two case studies. The first one focuses on the key properties of the derived control law, including its almost global asymptotic stability and its capability to converge to a predetermined velocity profile. In the second case-study, we demonstrate the advantages of the proposed pose-following control law in comparison to its predecessor, the pose-tracking control law.

\noindent \textit{Numerical implementation:} For all evaluations, the parameters are kept constant as $m=1\text{kg}$, $J =\text{diag}(0.01,0.01,0.01)\,\text{Kg}/\text{m}^{-2}$, $\hat{\bm{k}}_p=\hat{\bm{k}}_v=\hat{\bm{3}}$ and $k_\theta$ = 1.

\subsection{Almost global asymptotic stability on pose-following with velocity assignment}
\noindent The primary focus of this study is to validate the outcome of Theorem 2: almost global asymptotic stability for pose-following with velocity assignment. When doing so, we intend to verify that, regardless of the initial state, the pose of the rigid body converges to the geometric reference. To this end, we initialize the system at four distinct poses. In addition, we also want to show that this convergence is upheld irrespective of the velocity assignment. To achieve this goal, we evaluate each starting condition according to two distinct profiles, namely a slow one with $\theta_{vd,\text{slow}}=0.019$ and a fast one with $\theta_{vd,\text{fast}}=0.075$. As an exemplary geometric reference, we select the same three-dimensional curve as in \cite{kumar2017path} and to conform to the requirements of \eqref{eq:geom_ref}, we assign a moving frame to it. For this purpose, akin to~\cite{arrizabalaga2022spatial,arri2022spatially}, we rely on  \emph{Pythagorean Hodograph curves}, allowing us to overcome the singularities and discontinuities of the well-known \emph{Frenet-Serret} frame.


The motions resulting from applying the control law~\eqref{eq:U} -- with the pose-parameter control $\hat{U}_\theta$ defined as in Theorem 2 -- to the rigid body dynamics in~\eqref{eq:rigid_body_f} are depicted in the left side of Fig.~\ref{fig:gas_results}. The motions respective to the runs with the faster velocity profiles are shown by a continuous line and their respective orientations, while the ones related to the slow profiles are represented by dashed lines.

These motions manifest two noteworthy characteristics. The first one being that all of them demonstrate asymptotic convergence towards the geometric reference. The second characteristic, which aligns with common intuition, is that motions corresponding to the slower velocity profiles attain convergence at an earlier stage.

For a more comprehensive analysis, we direct attention to the purple case-study, which refers to the motion located at the top-left corner of Fig.~\ref{fig:gas_results}(a.1). A magnified view of this case is presented on the right-hand side of Fig.~\ref{fig:gas_results}(b.2). We hereby validate that the velocity of the pose-parameter, $\dot{\theta}(t)$, achieves convergence with the desired velocity profile, $\theta_{vd}$. To accomplish this, as demonstrated in Fig.~\ref{fig:gas_results}(b.3), we examined the convergence in not only slow (green) and fast (purple) constant velocity profiles but also in a sinusoidal profile (orange).

Lastly, we showcase the importance of taking care of the existence of two equilibrium points $\pm\hat{I}$, which in our approach is handled by the switching term $\lambda$~\eqref{eq:u_fb}. Within the same case-study as in the previous paragraph, we show that if this switching term is deactivated, the control law only converges to $\hat{I}$ resulting in unnecessarily lengthy and large motion. This can be visualized in the position and quaternion errors, as well as in the resultant motions colored in light gray in~Fig.\ref{fig:gas_results}(b.1-2).

\subsection{Comparison to pose-tracking}
\noindent Having analyzed the most relevant properties of the presented control law, in this second case-study we compare the performance of the proposed pose-following approach against pose-tracking~\cite{wang2013unit}. For this purpose, we pick a planar sinusoidal curve with a moving frame attached to it as a geometric reference. The task at hand consists of traversing the geometric reference from a zero-velocity pose. However, at the middle of the navigation a longitudinal and angular disturbance is introduced. To ensure a fair comparison, both the pose-tracking and pose-following have been tuned to ensure that the navigation time is the same if no disturbance occurs.

In this experiment the desired velocity profile function is chosen to be dependent on the distance to the geometric reference: $U_\theta(x_{\Gamma}(t)) = -k_\theta\left(\dot{\theta}(t) - \theta_{vd}(d_{e\,\perp}(x_{\Gamma}(t))\right)$\footnote{$d_{e,\perp}(x_{\Gamma}(t))$ is the transverse distance to the geometric reference.} Intuitively, if the system is far away from the reference, it slows down until it is close enough to increase the speed. This mapping is regarded as a tuning parameter that the user can tailor based on system properties and task at hand. In an illustrative manner, we design three variants: progressive (red), medium (orange) and conservative (blue). These velocity profiles alongside their associated motions can be visualized in Fig.~\ref{fig:disturbance_results}. 

When compared to pose-tracking (in magenta), two differences can be spotted. First, at the very beginning of the trajectory, the tracking method shows a small deviation from the reference. This is due to the fact that the rigid body initially is standing still and needs to catch up with the moving time-reference. In contrast the presented pose-following is aware of its initial state and progressively increases its velocity along the reference. Second, as soon as the disturbance is over, the additional degree of freedom inherited from augmenting the system allows all three variants to slow down and converge back to the geometric reference. This can clearly be visualized in the evolution of $\dot{\theta}$. As expected, the convergence rate directly correlates to how conservative the desired velocity profile mapping is. On the other hand, pose-tracking lacks this additional degree of freedom and has no choice but to catch up with the time-based reference, causing a large deviation error.

\section{CONCLUSIONS}\label{sec:conclusions}
\noindent In this paper we have formulated a unit dual quaternion-based pose-following control approach for rigid body dynamics. Initially, we have derived the equations of motion for the full pose error between the rigid body and the geometric reference in the form of a dual quaternion and dual twist. Subsequently, we have extended the original control law to account for nonlinearities arising from the introduction of auxiliary states associated with pose-following and designed the additional degree of freedom either to achieve convergence to a desired velocity profile or as a feedback mechanism. When doing so, we have also established almost global asymptotic stability. Lastly, we have numerically validated our findings with two illustrative simulations.

\newcommand{\BIBdecl}{\setlength{\itemsep}{0.015 em}} 
\bibliographystyle{IEEEtran}
\bibliography{path_following_with_dual_quaternions}

\ifcomment
\newpage
\begin{appendices}
\section{Derivation of dual twist error}\label{apend:dt_e_app}
\noindent Let $\hat{q}_e$ and $\hat{\omega}_e$, defined in eqs.~\eqref{eq:dq_error} and ~\eqref{eq:dte}, be the dual quaternion and twist errors between the pose of rigid body~\eqref{eq:rigid_body_dynamics} and the parametric geometric reference~\eqref{eq:geom_ref}. Further derivating the dual twist error as in~\eqref{eq:we} requires from the upcoming three lemmas:
\begin{lemma}\label{lemma:lemma1}
Let $q_e$ and $\bm{\omega}_e$ refer to the quaternion and twist errors. Then,
\begin{subequations}
    \begin{align*}
    \dot{q}_e(t) &= \frac{1}{2}\bm{\omega}_e(t)\circ\, q_e(t)\,,\\
    \bm{\omega}_e(t) &= \omega(t) + \dot{\theta}(t)\text{Ad}_{q_e(t)}\bm{\omega}_d^*(\theta(t))\,.
    \end{align*}
\end{subequations}
\end{lemma}
\begin{proof}
    Defining the quaternion error as $q_e(t) = q(t) \circ q_d^*(\theta(t))$, its time derivative is given by
    \begin{equation*}
        \dot{q}_e(t) = \dot{q}(t)\circ q_d^*(\theta(t)) + q(t)\circ \dot{q}_d^*(\theta(t))
    \end{equation*}
    Leveraging that $\dot{q}(t)=\frac{1}{2}\bm{\omega}(t)\circ q(t)$ and $\dot{q}^*(t)=\frac{1}{2}q^*(t)\circ\bm{\omega}^*(t)$, results in
    \begin{equation*}
    \begin{split}
    \dot{q}_e(t) =& \frac{1}{2}(\bm{\omega}(t)\circ q(t)\circ q_d^*(\theta(t)) +\\
    &\dot{\theta}(t)q(t)\circ q^*_d(\theta(t))\circ\bm{\omega}^*_d(\theta(t)))\,,
    \end{split}
    \end{equation*}
    and further simplifies into
    \begin{equation*}
    \dot{q}_e(t) = \frac{1}{2}\left(\bm{\omega}(t)\circ q_e(t) + \dot{\theta}(t)q_e(t)\circ\bm{\omega}^*_d(\theta(t)))\right)\,,
    \end{equation*}
    which is equivalent to
    \begin{equation*}
        \dot{q}_e(t) = \frac{1}{2}\underbrace{\left(\bm{\omega}(t) + \dot{\theta}(t)q_e(t)\circ\bm{\omega}_d^*(\theta(t))\circ q_e^*(t)\right)}_{\bm{\omega}_e(t)}\circ\, q_e(t)\,.
    \end{equation*}
\end{proof}
\begin{lemma}\label{lemma:lemma2}
Let $\bm{p}_e(t)$ refer to the position error as
\begin{equation*}
    \bm{p}_e(t) = \bm{p}(t) + \text{Ad}_{q_e(t)}\bm{p}_d^*(\theta(t))\,,
\end{equation*}
then
\begin{equation*}
    \begin{split}
    \dot{\bm{p}}_e(t) =&\, \dot{\bm{p}}(t) + \bm{\omega}_e(t)\cross \text{Ad}_{q_e(t)}\bm{p}_d^*(\theta(t)) +\\
    &\dot{\theta}(t)\text{Ad}_{q_e(t)}\mathring{\bm{p}}_d^*(\theta(t))\,.
    \end{split}
\end{equation*}
\end{lemma}
\begin{proof}
    Derivating the position error in time
    \begin{equation*}
        \begin{split}
        \dot{\bm{p}}_e(t) =&\, \bm{p}(t) + \dot{q}_e(t)\circ\bm{p}_d^*(\theta(t))\circ q_e^*(t) +\\
        &q_e(t)\circ \dot{\theta}(t)\mathring{\bm{p}}_d^*(\theta(t))\circ q_e^*(t) + \\
        &q_e(t)\circ\bm{p}_d^*(\theta(t))\circ \dot{q}_e^*(t)\,,
        \end{split}
    \end{equation*}
    which combined with Lemma~\ref{lemma:lemma1} leads to
    \begin{equation*}
        \begin{split}
        \dot{\bm{p}}_e(t) =&\, \bm{p}(t) + \frac{1}{2}\left(\bm{\omega}_e(t)\circ q_e(t) \circ \bm{p}_d^*(\theta(t))\circ q_e^*(t) +\right.\\
        &\left.q_e(t)\circ\bm{p}_d^*(\theta(t))\circ q_e^*(t)\circ \bm{\omega}_e^*(t)\right)+\\
        &q_e(t)\circ \dot{\theta}(t)\mathring{\bm{p}}_d^*(\theta(t))\circ q_e^*(t)\,,\\
        \end{split}        
    \end{equation*}
    and noticing that $\bm{\omega}_e^*(t) = -\bm{\omega}_e(t)$
    \begin{equation*}
        \begin{split}
        \dot{\bm{p}}_e(t) =&\, \bm{p}(t) + \frac{1}{2}\left(\bm{\omega}_e(t)\circ \text{Ad}_{q_e(t)}\bm{p}_d^*(\theta(t)) -\right.\\
        &\left.\text{Ad}_{q_e(t)}\bm{p}_d^*(\theta(t))\circ \bm{\omega}_e(t)\right)+\\
        &\dot{\theta}(t)\text{Ad}_{q_e(t)}\mathring{\bm{p}}_d^*(\theta(t))\,.\\
        \end{split}        
    \end{equation*}
    Given that for vector quaternions $q_1\circ q_2 = [-q_1q_2,\,q_1\cross q_2]$ and noticing that $\bm{\omega}_e(t)$ is perpendicular to $\text{Ad}_{q_e(t)}\bm{p}_d^*(\theta(t))$,
    \begin{equation*}
        \begin{split}
        \dot{\bm{p}}_e(t) =&\, \bm{p}(t) + \bm{\omega}_e(t)\circ \text{Ad}_{q_e(t)}\bm{p}_d^*(\theta(t))+\\
        &\dot{\theta}(t)\text{Ad}_{q_e(t)}\mathring{\bm{p}}_d^*(\theta(t))\,,\\
        \end{split}        
    \end{equation*}
    or equivalently,
    \begin{equation*}
        \begin{split}
        \dot{\bm{p}}_e(t) =&\, \bm{p}(t) + \bm{\omega}_e(t)\cross \text{Ad}_{q_e(t)}\bm{p}_d^*(\theta(t))+\\
        &\dot{\theta}(t)\text{Ad}_{q_e(t)}\mathring{\bm{p}}_d^*(\theta(t))\,.\\
        \end{split}        
    \end{equation*}    
\end{proof}
\begin{lemma}\label{lemma:lemma3}
The following statement is true
\begin{equation*}
    \begin{split}
    \bm{p}_e(t)\cross\bm{\omega}_e(t) =& \bm{p}(t)\cross\bm{\omega}(t) +\\
    &\dot{\theta}(t)\text{Ad}_{q_e(t)}\bm{\omega}_d^*(\theta(t)) +\\
    &\text{Ad}_{q_e(t)}\bm{\omega}_d^*(\theta(t)) \cross \bm{\omega}_e(t)
    \end{split}
\end{equation*}
\end{lemma}
\begin{proof}
    Expanding the cross product with the definitions in Lemmas~\ref{lemma:lemma1} and~\ref{lemma:lemma2},
    \begin{equation*}
        \begin{split}
            \bm{p}_e(t)\cross\bm{\omega}_e(t) =& \left[\bm{p}(t) + \text{Ad}_{q_e(t)}\bm{p}_d^*(\theta(t))\right]\cross\\
            &\left[\bm{\omega}(t) + \dot{\theta}(t)\text{Ad}_{q_e(t)}\bm{\omega}_d^*(\theta(t))\right]\,,
        \end{split}
    \end{equation*}
    which further simplifies into
    \begin{equation*}
        \begin{split}
            \bm{p}_e(t)\cross\bm{\omega}_e(t) =&\, \bm{p}(t)\cross\bm{\omega}(t) +\\
            &\dot{\theta}(t)\bm{p}(t)\cross\text{Ad}_{q_e(t)}\bm{\omega}_d^*(\theta(t)) +\\
            &\text{Ad}_{q_e(t)}\bm{p}_d^*(\theta(t))\cross\bm{\omega}_e(t)\,.
        \end{split}
    \end{equation*}
\end{proof}

\noindent Using these three Lemmas we state the following theorem:
\begin{theorem}\%label{lemma:lemma3}
The dual twist error is given by
\begin{equation*}
    \hat{\omega}_e(t) = \left[0,\bm{\omega}_e(t)\right] + \epsilon\left[0, \dot{\bm{p}}_e(t) + \bm{p}_e(t)\cross\bm{\omega}_e(t)\right]\,.
\end{equation*}
\end{theorem}
\begin{proof}
    The proof consists on combining the dual twist error in~\eqref{eq:dte}
    \begin{equation*}
        \hat{\omega}_e(t) = \hat{\omega}(t) + \dot{\theta}(t)\text{Ad}_{\hat{q}_e(t)}\hat{\omega}_d^*(\theta(t))\,.
    \end{equation*}    
    with Lemmas~\ref{lemma:lemma1}, ~\ref{lemma:lemma2} and~\ref{lemma:lemma3}. First, expanding the second term follows as
    \begin{equation*}
        \begin{split}
        &\text{Ad}_{\hat{q}_e(t)}\hat{\omega}_d^*(\theta(t)) = \begin{bmatrix}q_e(t)\\\frac{1}{2}\bm{p}_e(t)\circ q_e(t) \end{bmatrix}\circ\\
        &\begin{bmatrix}\omega_d^*(\theta(t))\\\left(\mathring{\bm{p}}_d(\theta(t)) + \bm{p}_d(\theta(t))\cross\bm{\omega}_d(\theta(t))\right)^*\end{bmatrix}\circ \begin{bmatrix}q_e^*(t)\\\frac{1}{2}q_e^*(t)\circ\bm{p}_e^*(t) \end{bmatrix}\,, 
        \end{split}
    \end{equation*}
    which, after some derivations, can be further simplified into
\begin{equation*}
        \begin{split}
        &\text{Ad}_{\hat{q}_e(t)}\hat{\omega}_d^*(\theta(t)) =\\
        &\begin{bmatrix}
            \text{Ad}_{\hat{q}_e(t)}\bm{\omega}_d^*(\theta(t))\\
            \bm{p}(t)\cross \text{Ad}_{\hat{q}_e(t)}\bm{\omega}_d^*(\theta(t)) + \text{Ad}_{\hat{q}_e(t)}\mathring{\bm{p}}_d^*(\theta(t))
        \end{bmatrix}\,.
        \end{split}
\end{equation*}
Getting back to the dual twist error in~\eqref{eq:dte} and replacing its second term by the expression above, and the first one by the definition of the dual twist in~\eqref{eq:dt}, leads to
\begin{equation*}
    \begin{split}
    &\hat{\omega}_e(t) = \bm{\omega}(t) + \dot{\theta}(t)\text{Ad}_{\hat{q}_e(t)}\bm{\omega}_d^*(\theta(t)) +\\
    &\epsilon[\dot{\bm{p}}(t) + \bm{p}(t)\cross\bm{\omega}(t)+\\
    &\dot{\theta}(t)\left(\bm{p}(t)\cross\text{Ad}_{\hat{q}_e(t)}\bm{\omega}_d^*(\theta(t)) + \text{Ad}_{\hat{q}_e(t)}\mathring{\bm{p}}_d^*(\theta(t))\right)]\,.
    \end{split}
\end{equation*}
Combining the real part with Lemma~\ref{lemma:lemma1} and leveraging the following statement resulting from Lemmas~\ref{lemma:lemma2} and~\ref{lemma:lemma3} within the dual part results in
\begin{equation*}
\begin{split}
& \dot{\bm{p}}_e(t) + \bm{p}_e(t) \cross \bm{\omega}_e(t) = \dot{\bm{p}}(t) + \bm{p}(t)\cross\bm{\omega}(t)+\\
&\dot{\theta}(t) \left(\text{Ad}_{q_e(t)}\mathring{\bm{p}}_d^*(\theta(t)) + \bm{p}(t)\cross \text{Ad}_{q_e(t)}\bm{\omega}_d^*(\theta(t))\right)\,,
\end{split}
\end{equation*}
and therefore, the dual twist error can be expressed as
\begin{equation*}
    \hat{\omega}_e(t) = \left[0,\bm{\omega}_e(t)\right] + \epsilon\left[0, \dot{\bm{p}}_e(t) + \bm{p}_e(t)\cross\bm{\omega}_e(t)\right]\,.
\end{equation*}
\end{proof}
\end{appendices}
\fi

\end{document}